\pdfoutput=1
\documentclass[conference]{IEEEtran}
\usepackage{times}

\usepackage[numbers,sort&compress]{natbib}
\usepackage{multicol}
\usepackage[bookmarks=true]{hyperref}
\usepackage{amssymb, amsmath, amsthm}
\usepackage{algorithm,algorithmicx,algpseudocode}
\usepackage{comment}
\usepackage{caption}
\usepackage{subcaption}
\usepackage{graphicx}
\usepackage{multicol}
\usepackage{float}
\usepackage{siunitx}
\usepackage[normalem]{ulem}
\graphicspath{{figures/}}

\usepackage[dvipsnames]{xcolor}

\usepackage{amsthm}  % Ensure we are using amsthm.
\usepackage{amsmath}
\usepackage{bm}

\newtheorem{theorem}{Theorem}
\newtheorem{proposition}{Proposition}

\newtheorem{definition}{Definition}
\newtheorem{assumption}{Assumption}

\usepackage{ifthen}
\newboolean{include-notes}
\setboolean{include-notes}{false}  % Set boolean to false to hide all notes and "to-be-removed" text.

\usepackage{xcolor}
\newcommand{\jfnote}[1]{\ifthenelse{\boolean{include-notes}}{\textcolor{orange}{\textbf{Jaime:} #1}}{}}
\newcommand{\zznote}[1]{\ifthenelse{\boolean{include-notes}}{\textcolor{teal}{\textbf{Zixu:} #1}}{}}
\newcommand{\remove}[1]{\ifthenelse{\boolean{include-notes}}{\textcolor{red}{\sout{#1}}}{}}

\usepackage{lipsum}

\newcommand{\reals}{\mathbb{R}}

\newcommand{\traj}{{\mathbf{x}}}%{{\xi}}%
\newcommand{\csig}{{\mathbf{u}}}

\newcommand{\xset}{{\mathcal{C}}}  %{{\reals^\nx}}
\newcommand{\cset}{{\mathcal{U}}}

\newcommand{\wspace}{{\mathcal{W}}}  % workspace

\newcommand{\dyn}{{f}}
\newcommand{\policy}{{\pi}}

\newcommand{\cstrat}{{\bm{\upsilon}}}

\newcommand{\safeset}{{\Omega}}

\newcommand{\failure}{{\mathcal{F}}}

\newcommand{\reachavoid}{{\mathcal{RA}}}

\newcommand{\fras}{{\overrightarrow{\reachavoid}}}

\newcommand{\hidden}{{\mathcal{H}}}
\newcommand{\obj}{{\tilde{o}}}
\newcommand{\obst}{{\mathcal{O}}}
\newcommand{\capture}{{\text{Cap}}}
\newcommand{\capturebasin}{{\mathcal{B}}}
\newcommand{\fhs}{{\overrightarrow{\hidden}}}

\newcommand{\danger}{{\mathcal{D}}} % Danger (failure) zone in joint state space
\newcommand{\footprint}{{\phi}}    % Robot footprint (subindex for space)

\newcommand{\fov}{{\mathcal{Z}}} % Field of view (sensor)
\begin{document}

% paper title
\title{Safe Occlusion-aware Autonomous Driving via Game-Theoretic Active Perception}

% You will get a Paper-ID when submitting a pdf file to the conference system
%\author{Author Names Omitted for Anonymous Review. Paper-ID 235}

% \author{\IEEEauthorblockN{Zixu Zhang}
% \IEEEauthorblockA{Department of Electrical Engineering\\
% Princeton University\\
% Princeton, New Jeresy 08544\\
% Email: zixuz@princeton.edu}
% \and
% \IEEEauthorblockN{Jaime F. Fisac}
% \IEEEauthorblockA{Department of Electrical Engineering\\
% Princeton University\\
% Princeton, New Jeresy 08544\\
% Email: jfisac@princeton.edu}}

\author{\IEEEauthorblockN{Zixu Zhang and Jaime F. Fisac}
	\IEEEauthorblockA{Department of Electrical and Computer Engineering\\
		Princeton University, Princeton, NJ 08544\\
		Email: zixuz@princeton.edu;~jfisac@princeton.edu}}
	
\maketitle

\begin{abstract}
Autonomous vehicles interacting with other traffic participants heavily rely on the perception and prediction of other agents' behaviors to plan safe trajectories. However, as occlusions limit the vehicle's perception ability, reasoning about potential hazards beyond the field of view is one of the most challenging issues in developing autonomous driving systems. This paper introduces a novel analytical approach that poses safe trajectory planning under occlusions as a hybrid zero-sum dynamic game between the autonomous vehicle (evader) and an initially hidden traffic participant (pursuer). Due to occlusions, the pursuer's state is initially unknown to the evader and may later be discovered by the vehicle's sensors. The analysis yields optimal strategies for both players as well as the set of initial conditions from which the autonomous vehicle is guaranteed to avoid collisions. We leverage this theoretical result to develop a novel trajectory planning framework for autonomous driving that provides worst-case safety guarantees while minimizing conservativeness by accounting for the vehicle's ability to actively avoid other road users as soon as they are detected in future observations. Our framework is agnostic to the driving environment and suitable for various motion planners. We demonstrate our algorithm on challenging urban and highway driving scenarios using the open-source CARLA simulator.

\end{abstract}

\IEEEpeerreviewmaketitle

\section{Introduction}
 
Line-of-sight sensors, such as cameras, radars, and LiDARs, are widely used on autonomous vehicles and other mobile robots to perceive surroundings and make real-time decisions. The robustness of trajectory planning in dynamic environments,
such as highway driving (Fig.~\ref{fig: pitch_figure}), has primarily depended on predicting and analyzing other agents' future actions (we refer the reader to recent thorough surveys~\cite{Paden2016_av_motion_survey,Rudenko2020HumanPrediction} on these topics). In everyday driving scenarios, occlusions are ubiquitous and may be encountered in various forms (common types are shown in Fig. \ref{fig: occlusion}). They limit the sensors' perception abilities and introduce uncertainty in motion predictions. For example, some occlusions are due to the road's geometric nature, where landscapes and buildings block the view of potential road hazards ahead. Human drivers usually slow down in these situations, to earn more reaction time and allow an eventual stop if needed. Other occlusion scenarios also involve dynamic agents, where more complicated interactions than solely slowing down are necessary to maintain safety and achieve trajectory efficiency.

% \begin{figure}[t!]
%     \centering
%     \begin{subfigure}[b]{0.49\textwidth}
%         \centering
%         \includegraphics[width=3.45in]{figures/pitch_figure_t0.eps}
%     \end{subfigure}
%     \begin{subfigure}[b]{0.49\textwidth}
%         \centering
%         \includegraphics[trim=0 0 0 25, clip, width=3.45in]{figures/pitch_figure.eps}
%     \end{subfigure}
% \captionsetup{belowskip=-4pt}
% \caption{The autonomous vehicle (evader) at $x^{t_0}_e$ verifies whether it can safely overtake the slow-moving truck by preceding the trajectory towards a waypoint $x^{t_n}_e$ and avoid future collision with a potential hidden agent (pursuer) outside the current field of view $\fov^{t_0}$.
% Our framework utilizes a game-theoretical approach to identify the existence of a closed-loop evasive strategy if an oncoming agent is observed at $t_n$, by checking the intersection between the danger zone $\danger(x_e^{t_n})$ of the proposed waypoint and the forward reach-avoid set
% $\vec{\hidden}(t_n;\hidden^{t_0}_\obj, \traj_e)$ of all possible hidden agents that may be observed at $t_n$. In the example above, the proposed trajectory is not safe, as the danger zone intersects with the forward reach-avoid set.}
% \label{fig: pitch_figure}
% \end{figure}

\begin{figure}[t!]
    \centering
    \begin{subfigure}[b]{0.49\textwidth}
        \centering
        \includegraphics[width=3.5in]{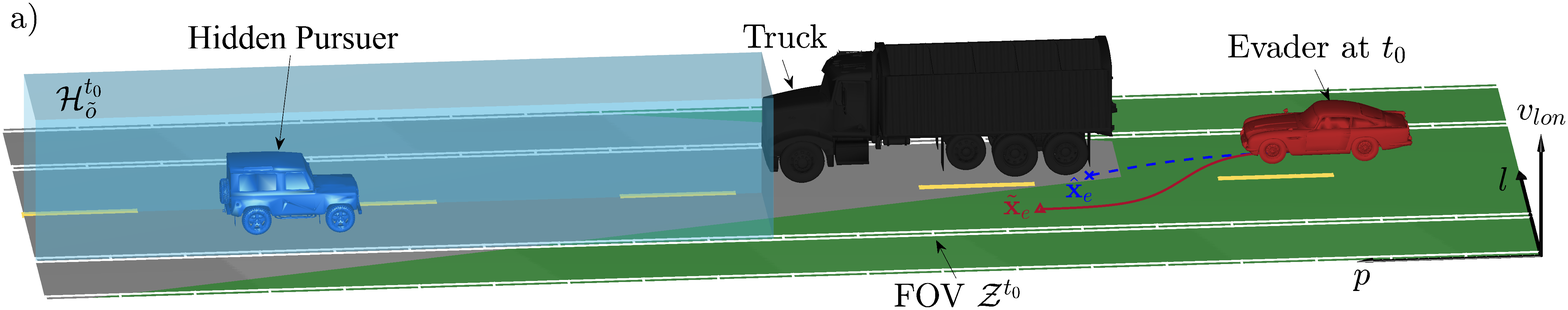}
    \end{subfigure}
    \begin{subfigure}[b]{0.49\textwidth}
        \centering
        \includegraphics[width=3.5in]{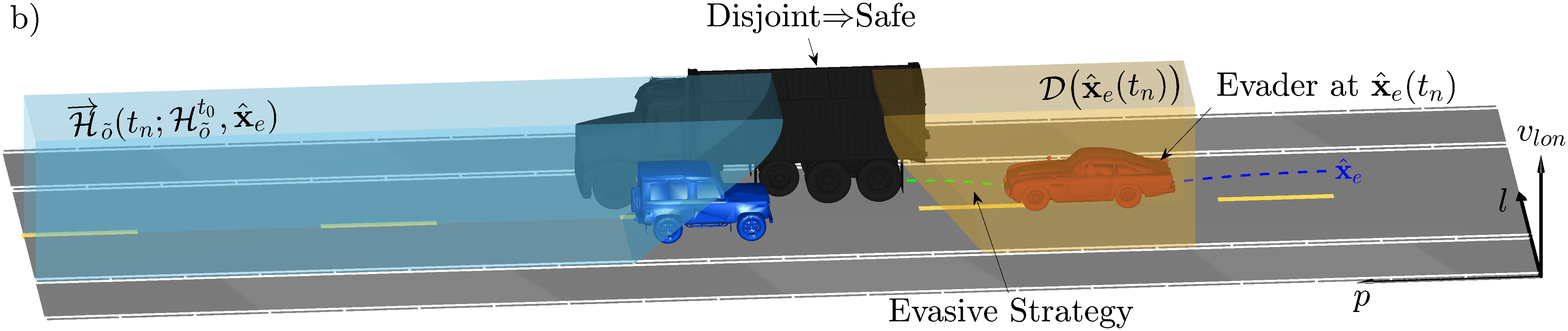}
    \end{subfigure}
    \begin{subfigure}[b]{0.49\textwidth}
        \centering
        \includegraphics[width=3.5in]{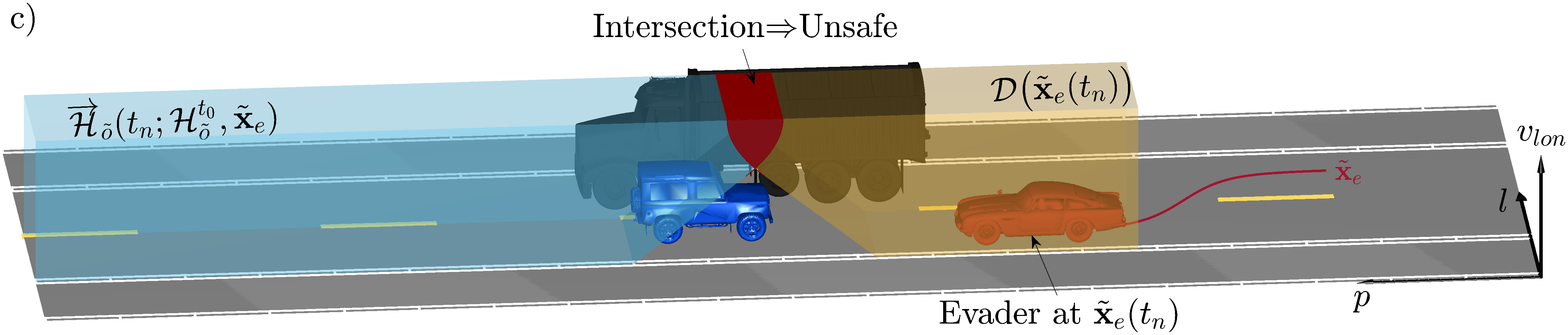}
    \end{subfigure}
\captionsetup{belowskip=-4pt}
\caption{The autonomous vehicle (evader) verifies the safety of two potential trajectories $\tilde{\traj}_e$ and $\hat{\traj}_e$ against collisions with a hidden agent (pursuer) outside the current field of view $\fov^{t_0}$.
Our framework utilizes a game-theoretic approach to determine the existence of a closed-loop evasive strategy if an agent is observed at $t_n$, by checking the intersection between the danger zone $\danger(\traj_e(t_n))$ of the proposed waypoint and the forward hidden set $\vec{\hidden}(t_n;\hidden^{t_0}_\obj, \traj_e)$. Using this approach, we can determine that the trajectory $\hat{\traj}_e$ shown in b) is safe, while $\tilde{\traj}_e$ shown in c) is not safe due to the intersection between two sets.  Supplementary materials can be found in \texttt{\url{https://saferobotics.princeton.edu/research/occlusion_hybrid_game}}~.}
\label{fig: pitch_figure}
\end{figure}

For autonomous vehicles to operate safely, risks outside their field of view must be evaluated while planning for future actions. Previous works \cite{koschi2020set, Orzechowski2018OcclusionSet, Yu2019risk, yu2019occlusion} utilize forward reachability analysis to over-approximate potential occluded objects' future occupancy based on the current sensor observations and generate collision-free trajectories by avoiding the entire forward reachable set (FRS) over the planning horizon. However, this approach ignores the autonomous agent's future ability to respond to objects, which are currently occluded but may later become visible through new sensor observations. Furthermore, especially when considering potential dynamic agents outside of the field of view, their FRS grows rapidly as the planning time horizon increases. Thus, while FRS-based method provides a sufficient condition to achieve safe behaviors in the presence of occlusions, this may cause the ego agent to choose suboptimal and unnecessarily conservative actions, which can negatively impact driving performance, as well as trust and perceived safety by the public~\cite{kaur2018trust}.

\begin{figure*}[ht]
	\centering
	\begin{subfigure}[b]{0.3\textwidth}
		\centering
		\includegraphics[height=1.5in]{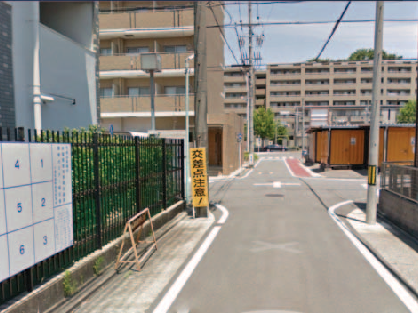}
		\caption{A blind intersection}
	\end{subfigure}	
\hfill
	\begin{subfigure}[b]{0.3\textwidth}
		\centering
		\includegraphics[height=1.5in]{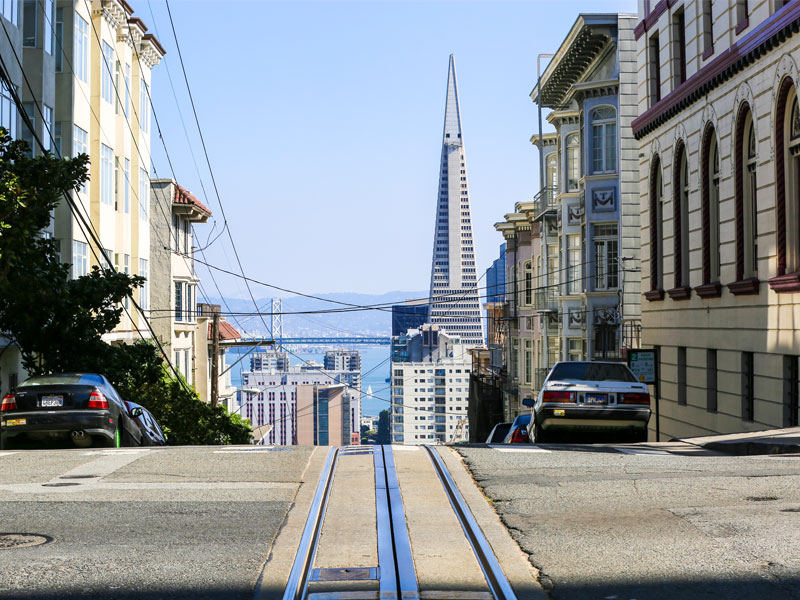}
		\caption{A blind summit in San Fransisco}
	\end{subfigure}	
\hfill 
	\begin{subfigure}[b]{0.3\textwidth}
		\centering
		\includegraphics[height=1.5in]{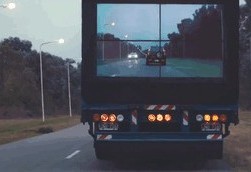}
		\caption{Overtaking a truck}
		\label{fig: overtake_truck}
	\end{subfigure}	
\captionsetup{belowskip = -6pt}
\caption{Common occlusion scenarios that create challenges and safety risks for autonomous vehicles}
\label{fig: occlusion}		
\end{figure*}

A reasonable strategy to overtake a slow-moving truck on an undivided highway, as shown in Fig. \ref{fig: pitch_figure} and \ref{fig: overtake_truck}, is first to gain the visibility of the incoming traffic by slowing down and cautiously moving towards the oncoming lane. During this process, the ego vehicle must reason about whether it can safely move back to the original lane if any objects show up from the occluded regions. While slowly edging sideways does not maximize the instantaneous progress, it helps to gather additional information regarding the occluded regions, opening up the possibility of overtaking, and potentially making more progress over a longer time horizon. Inspired by the above observations, we set out to formally analyze the decision-making problem faced by an autonomous system navigating a partially occluded environment, considering the worst-case existence and behavior of potential unobserved objects defined within the Operational Design Domain (ODD) \cite[Ch.~2]{thorn2018_ODD} and, critically, accounting for the ability to detect and actively avoid them once they are observed.

For each potential object, we find that this decision-making problem takes the form of a two-mode zero-sum hybrid dynamic game between an evader, the autonomous vehicle,  with limited information, and an initially hidden pursuer with perfect state information.
The terminal mode of the game is a closed-loop feedback game in which both players are able to observe each other's state at each time instant and adapt their actions strategically. The first mode of the game, in which the pursuer is invisible to the evader, can be posed as an open-loop ``trajectory game'', in which the pursuer may choose its worst-case trajectory in response to \emph{each} possible trajectory chosen by the evader. While the evader is at a strong informational disadvantage in the first mode of the game, it can plan strategically to force a transition to the second mode by having the invisible pursuer fall into its field of view. The pursuer can only win the overall game (force the evader to collide or violate constraints) if it enters the \emph{capture basin} of the second mode (from which the evader cannot escape capture) without first being detected.

    Based on this theoretical analysis we propose a novel occlusion-aware trajectory planning framework for autonomous vehicles that leverages this two-mode dynamic game structure.
The closed-loop planner solves a reach-avoid game \cite{bacsar1998dynamic, HO1965_pursuit_evasion,isaacs1999differential}, computing an evasive feedback policy towards an invariant safe set while avoiding collision with both observed and hidden objects. The open-loop game solution is implemented by a receding horizon planner, which generates optimal trajectories according to some cost function (e.g., progress, comfort, and efficiency), based on (possibly conservative) predictions of objects currently inside the vehicle's field of view and under the assumption (for optimization purposes) that no hidden objects will appear.\footnote{%
        This assumption is made for simplicity of the performance optimization and may be replaced by an alternative criterion, e.g. a probabilistic model of hidden object existence. We note that this only affects performance, and has no bearing on safety guarantees, which are our focus in this work.}
To provide a safety guarantee, we introduce a constraint on the future field of view, obtained from the game-theoretic analysis, to ensure that the closed-loop planner will find a collision-free evasive strategy under the worst-case scenario from all waypoints along the planned trajectory.

    The remainder of the paper is structured as follows:
    In Section \ref{sec:related_work}, we place our contribution in the context of related efforts in the existing literature.
    In Section \ref{sec:formulation}, we introduce some important theoretical building blocks.
    In Section \ref{sec:game}, we formalize the problem of safe planning under occlusions as a two-mode hybrid dynamic game, and characterize its solution in terms of optimal strategies and winning conditions.
    In Section \ref{sec:framework}, we introduce the algorithmic trajectory planning framework based on the game-theoretic analysis.
    In Section \ref{sec:results}, we evaluate our proposed framework in a diverse set of representative driving scenarios using the open-source CARLA simulator~\cite{Dosovitskiy17_Carla} and demonstrate the benefits of our hybrid open/closed-loop game-theoretic solution with respect to recently proposed, more conservative methods.
    Finally, in Section \ref{sec:conclusion}, we provide a closing discussion and briefly consider promising directions for future work.

\section{Related Work}\label{sec:related_work}
Perception-aware motion planning has been studied for several decades. In particular, the risk of potential objects beyond the sensor's field of view has received attentions from the robotics and autonomous vehicle communities recently. Early works \cite{chung2009_safe, bouraine2012provably} plan trajectories for mobile robots in the free space represented by a shrunk field of view considering objects that may emerge from occlusions. To formally analyze the collision avoidance behavior with occluded objects, \citet{bouraine2014_passPMP} propose the braking Inevitable Collision States to achieve a passive-Safety guarantee~\cite{macek2009passive_safe}. \citet{Nager2019computational} employs the above ideas in the autonomous driving scenarios by predicting different traffic participants' future occupancy in the current field of view and avoid their braking Inevitable Collision States. However, the passive safety is a weak guarantee since other traffic participants may inevitably collide with the stopped autonomous vehicle. 

Recent works \cite{koschi2020set,Orzechowski2018OcclusionSet, Yu2019risk, yu2019occlusion, lee2017collision} tackle the risk assessment problem under occlusion by introducing ``phantom" objects beyond the field of view and predicting their future states through the forward reachability analysis. Yu et al.~\cite{Yu2019risk, yu2019occlusion} apply a sampling method to represent potentially occluded objects' states as particles. Bidirectional reachability analysis is used to determine the distribution of dangerous ``phantom" objects. While their methods efficiently represent the distribution of hidden obstacles, they do not provide guarantees to avoid collisions. 

To achieve safety guarantee, \citet{Orzechowski2018OcclusionSet} over-approximate the hidden objects' forward reachable set (FRS) by assuming an object may emerging from occlusion with the maximum possible velocity at any moment. A trajectory is then planned to avoid the intersection with the FRS. To improve FRS approximation, \citet{neel2020improving} use sequential observations to update the velocity bound that an object can emerge from the occlusion. \citet{koschi2020set} also propose a set-based approach to approximate the FRS of occluded objects by considering road structure and traffic rules. While the added structure helps reducing conservativeness, these approaches still ignore the autonomous vehicle's future ability to respond to the other road participants' presence and actions once new sensor observations are obtained. Therefore, they require the trajectory to avoid the computed FRS based on current information as a sufficient condition for safety, and the resulting behaviors are often overly conservative. %is not optimal.

Considering autonomous vehicle's future ability of sensing and reaction, previous works \cite{brechtel2014probabilistic, bouton2018scalable, Hubmann2019POMDP} formulate the occlusion-aware trajectory planning as a Partially Observable Markov Decision Process (POMDP) and carefully design the observation model in the belief state to take occlusion into account. While POMDP solutions can generate less conservative strategies for autonomous driving under occlusion, safety guarantees are not provided, and a reduction in the state-space is required to make the algorithm computationally tractable. 

The idea of active perception \cite{bajcsy1988active, bajcsy2018revisiting, Unterholzner2012_active_perception} attempts to close the loop between robot perception and motion planning by generating trajectories that will explore the environment and gain information to reduce uncertainty and improve future planning. Some recent works \cite{andersen2017max_visibility, wang2020generating} also tackle the occlusion-aware planning problem through active perception by incorporating the prediction of the autonomous vehicle's future field-of-view into motion planning. \citet{andersen2017max_visibility} solve the truck-overtaking problem through an MPC trajectory planner guided by a finite state-machine to determine overtaking behavior and maximize visibility. \citet{wang2020generating} propose a visibility risk metric based on predicted field-of-view to penalize the trajectories with low visibility and high velocity when approaching an occluded road section. While their results show improved the trajectory efficiency through exploration behaviors that increase visibility, their safety guarantees still rely on avoiding the over entire FRS of hidden objects.

\section{Preliminaries}\label{sec:formulation}
\subsection{System Dynamics and Reachable Sets}
%% Define State, Control, Dynamics

We consider an ego agent $e$ and other dynamic objects $o_1, o_2, \cdots$ in a three-dimensional environment $\wspace\subseteq\mathbb{R}^3$.
% Without loss of generality, we assume the agents start operation at $t=0$.
The evolution of each agent's state $x_i^t\in\mathbb{R}^{n_i}$ is governed by continuous-time dynamics
\begin{equation}
    \dot{x}^{t}_i = \dyn_i(x^t_i,u^t_i), \quad i\in\{e, o_1, o_2, \hdots\},
    \label{eq:dynamics}
\end{equation}
where $u^t_i\in\cset_i\subseteq\mathbb{R}^{m_i}$ is the agent's control input. The agent's allowable workspace $\wspace$ and input set $\cset_i$ are defined by the Operational Design Domain.
%% Define Trajectory

The input sets $\cset_i$ are assumed compact, and the dynamics $\dyn_i:\mathbb{R}^{n_i} \times \cset_i \to \mathbb{R}^{n_i}$ are assumed bounded, continuous in $u_i$ and Lipschitz-continuous in $x_i$.
Under these assumptions, for any initial agent configuration $x_i^0\in\mathbb{R}^{n_i}$ and measurable control signal ${\csig_i:[0, \infty)\rightarrow \cset_i}$, there exists a unique, continuous state trajectory $\traj_i(\cdot;0,x_i^0,\csig_i)$ satisfying (cf.~\cite[Ch.~2]{coddington1955theory}):
\begin{subequations}\label{eq:traj}
\begin{align}
\traj_i(0;0,x^0_i, \csig_i) =\,& x_i^0 \label{eq:traj_ic}\\
\dot{\traj}_i(t;0,x^0_i, \csig_i) =\,& f_i(\traj_i(t;0,x^0_i, \csig_i), \csig_i(t)),\, \text{a.e. } t\ge 0.
    \label{eq:traj_dyn}
\end{align}
\end{subequations}

To formally analyze the safe interaction between agents, we will often need to reason about \emph{constrained} reachable sets, in which we require that agents avoid entering certain (possibly time-varying) forbidden regions of their state space, which may be encoded through a set-valued map $\failure_i:[0,\infty)\rightrightarrows\mathbb{R}^{n_i}$,
with $\failure_i(\tau)\subseteq\mathbb{R}^{n_i}$ for each~$\tau$.
% These sets are denoted \emph{reach-avoid sets}.
\begin{definition}[Forward reach-avoid set]\label{def:FRAS}
    An agent $i$'s \textit{forward reach-avoid set} at time $t\ge t_0$ from an initial set $\xset_i^{t_0}$ under failure set $\failure_i(\cdot)$ is the set of all states that agent $i$ can reach at time $t$ starting from some state $x_i^{t_0}\in\xset_i^{t_0}$ at time $t_0$
    while avoiding the set $\failure_i(\tau)$ at all intermediate times~$\tau$:
	\begin{align*}
	    \fras_i(t;t_0,\xset^{t_0}_i\!\!,\failure_i) \!=\! \{&x^t_i \!\mid\! \exists x_i^{t_0}\!\in\xset_i^{t_0}, \exists \csig_i,
	    \traj(t; t_0, x_i^{t_0}, \csig_i)\!=\!x^t_i \\
	    &\wedge \traj(\tau; t_0, x_i^{t_0}, \csig_i)\not\in\failure_i(\tau)\, \forall \tau\in[t_0,t] \}
	\end{align*}
\end{definition}

\subsection{Observations and Field of View}
%% Define Projection between configuration space and workspace and FOV
    
The set-valued \emph{footprint} $\footprint_i:\mathbb{R}^{n_i}\rightrightarrows\wspace$ maps agent $i$'s state $x_i$ to its occupied workspace $\footprint_i(x_i)\subseteq\wspace$.
We will also denote by $\footprint_i^{-1}$ the inverse image induced by the footprint map: $\footprint_i^{-1}(M):=\{x_i\mid \footprint_i(x_i)\cap M\neq\emptyset\}$.
The occupied workspace by all non-ego objects at time $t$ is then given by
\begin{equation}
   \obst(t):=\bigcup_{i}\footprint_i(s^t_i)\subseteq\wspace, \quad i\in\{o_1, o_2, \cdots\}
   \label{eq:occupied_workspace}
\end{equation}
For compactness, we will often write this set as $\obst^t$.

We assume that the ego agent is equipped with a sensor (or suite of sensors) that enables it to accurately determine the presence or absence of other agents within certain well-observed regions of the workspace.
For the purposes of our analysis, we refer to this overall region as the ego agent's \textit{field of view} (FOV), denoted $\fov(x_e^t,\obst^t)\subseteq\wspace$.
The variation of the FOV over time as a function of the ego agent's configuration $x_e^t$ and the footprints of dynamic objects in the environment $\obst^t$ will be central to our analysis.

For example, an autonomous vehicle's perception system, in the context of this analysis, may be abstractly modeled through a line-of-sight sensor with range $r$, centered around the position $p_e^t(x_e^t)$ of the vehicle. This sensor's field of view $\fov(x_e^t,\obst^t)$ is then a set of points $y'\in\wspace$, such that $||p_e^t-y'||\leq r$ and the line segment connecting $p^t$ and $y'$ does not intersect with $\mathcal{O}^t$ (excepting end points).

Moreover, at the moment when the ego agent starts operating, it estimates a set of hypothetical hidden agents' states:
\begin{equation}\label{eq:hidden_set}
    \hidden^0_\obj :=
    \footprint_\obj^{-1}\big(\fov(x_e^0,\obst^0)\cup \obst^0\big)^c
\end{equation}
In the later section, we will discuss how the hidden set $\hidden^0_\obj$ is updated through observations and reachability analysis.

\section{The Hidden Pursuer Hybrid Game}\label{sec:game}
The theoretical contribution in this work is characterizing the least-restrictive guarantees for safe trajectory planning under nondeterministic (worst-case) uncertainty and dynamic environment occlusions through the solution to a hybrid differential game combining open-loop and closed-loop information patterns.

We focus on the problem of strategically detecting and avoiding a single (potentially) adversarial hidden object, in the midst of an arbitrary number of static and dynamic obstacles with bounded motion predictions, and note that the theoretical analysis may be extended to multiple simultaneous hidden adversaries, at the cost of increased complexity.
\begin{assumption}[Separation of hidden adversaries.]\label{assu:decouple}
    There exists at most one unobserved pursuer for the duration of the differential game.
\end{assumption}

We clarify that this assumption does not render our results inapplicable to scenarios with multiple hidden objects, which are common in driving settings.
Rather, Assumption~\ref{assu:decouple} restricts the ego vehicle to \emph{strategically} plan to detect and avoid one hidden pursuer at any given planning cycle.
Any additional potentially-existing objects can be represented as dynamic obstacles whose entire forward-reachable sets must be avoided, following the current state-of-the-art approach.
Our framework therefore enables reducing conservativeness around strategically important occlusions,
%by replacing the standard forward-reachable set avoidance requirement by the less stringent game-theoretic treatment presented here,
while maintaining safety guarantees with respect to any number of objects.

The problem at hand is then a two-player zero-sum differential game of pursuit~\cite{isaacs1951pursuit,isaacs1954differential}, in which the pursuer~$\obj$ seeks to capture (result in collision with) the evader~$e$, who in turn wishes to avoid capture.
The hidden pursuer's initial state $x_\obj^0$ may be anywhere within an initial existence set $\hidden_\obj^0$.
The evader and pursuer are each required to avoid a time-varying obstacle set $\obst_e:[0,T]\rightrightarrows\wspace$, $\obst_\obj:[0,T]\rightrightarrows\wspace$, which may be different\footnote{For example, we may require our ego vehicle, but not the hidden pursuer, to avoid other potentially existing objects.}.
The outcome of the game is binary, with the pursuer winning if it achieves capture, by driving the joint state $(x_e, x_\obj)$ into a capture set $\capture\subset\reals^{n_e}\times\reals^{n_\obj}$, or if the evader violates its obstacle-induced constraint;
the evader wins the game if the pursuer violates its obstacle constraint or if capture does not take place within the horizon $T$ of the game (in practice we will often consider the infinite-horizon case $T\to\infty$).

While the general solution of this game when both players have perfect state information is well studied (cf.~\cite{fisac2015pursuitevasiondefense,Fisac2015ReachAvoidGame}), our problem of interest introduces an important new consideration: the information structure of the game depends on the state history. In particular:
\begin{itemize}
    \item The pursuer $\obj$ has perfect state feedback information, i.e. can observe the full state $(x_e,x_\obj)$ at each time.
    \item The evader $e$ can only observe its own state $x_e^t$ \emph{until} the pursuer enters its field of view, i.e. $\footprint_\obj(x_\obj^t)\in\fov(x_e^t,\obst^t)$; after such an event, the evader also has perfect state feedback information.
\end{itemize}

Because of this history-dependent information structure, it is not possible to formulate this problem as either a closed-loop feedback game or an open-loop trajectory game~\cite[Ch.~8]{bacsar1998dynamic}.
However, we can instead decompose its solution through a two-mode hybrid differential game~\cite{tomlin2000game} with an initial open-loop mode and a terminal closed-loop mode.
As a result of the dynamic programming principle, the solution to the initial game depends on that of the terminal game; we will therefore consider them in backward order.

\subsection{Closed-Loop Game: Perfect-Information Pursuit-Evasion}
\label{sec: closed-loop}
In the terminal game, the pursuer has been detected by the evader, who can now adapt its evasive strategy to the observed pursuer motion over time.
This is a well-studied problem, corresponding to the classical \emph{games of pursuit} first studied by Isaacs in the 1950s~\cite{isaacs1951pursuit,isaacs1954differential}.
The general setting in which players are required to avoid dynamic obstacles was studied in~\cite{Fisac2015ReachAvoidGame,fisac2015pursuitevasiondefense}.

The winning set for the pursuer, or \emph{capture basin}, at time~$t$ is the set of joint conditions from which, regardless of the evader's strategy, either the pursuer can enforce capture or the evader otherwise collides, without the pursuer violating its obstacle constraints. Formally:
\begin{align} \label{eq:capture_basin}
    \capturebasin(t) = \Big\{&(x^t_e, x^t_\obj)\mid  \exists \cstrat_\obj,~\forall \csig_e, \exists\tau\geq t \notag\\
    &\Big(
    \footprint_e\big(\traj_{e}(\tau;t,x^t_e,\csig_e)\big)\!\cap \footprint_\obj\big(\traj_\obj(\tau;t,x_\obj^t, \cstrat_\obj[\csig_e])\big) \!\neq\! \emptyset
    \notag\\
    &\quad\vee \footprint_e\big(\traj_e(\tau;t,x_e^t, \csig_e)\big)\cap\obst_e(\tau)\neq\emptyset\Big) \\
    &\wedge \forall s\in[t,\tau], \footprint_\obj\big(\traj_\obj(s;t,x_\obj^t, \cstrat_\obj[\csig_e])\big)\cap\obst_\obj(s)\!=\!\emptyset\notag
    \Big\}
\end{align}
where $\cstrat_\obj$ is a non-anticipative strategy for the pursuer~\cite{evans1984differential};
this is equivalent to both players performing state feedback\footnote{The equivalence is due to \emph{Isaacs' condition}; we direct interested readers to~\cite[Ch.~8]{bacsar1998dynamic}.}.

The capture basin $\capturebasin: [0,T] \rightrightarrows \reals^{n_e} \times \reals^{n_\obj}$ and players' feedback strategies $\policy_e^t, \policy_\obj^t$ are encoded by the solution to a Hamilton-Jacobi-Isaacs variational inequality~\cite{fisac2015pursuitevasiondefense}. In practice, the numerical solution (which provides a necessary and sufficient condition for safety) may be computationally prohibitive for real-time planning in changing environments;
however, as we will see in Section~\ref{sec:framework}, conservative approximations can be efficiently computed, yielding a sufficient safety condition for the evader.

A critical observation is that in this closed-loop feedback game the evader need not have any single available trajectory that will avoid \emph{all} possible motions of the pursuer (i.e. its forward-reachable set) in order to avoid capture; instead, it can \emph{adapt} its course of action to the unfolding pursuer trajectory, since different pursuer futures are mutually exclusive.
This is a key distinction between our analysis and existing methods based on avoidance of the entire forward-reachable set.

% \textcolor{blue}{Note: here it'll be \emph{very} important to stress that our analysis is \emph{less} conservative than those doing FRS avoidance.}
\subsection{Open-Loop Game: Preventing an Invisible Pursuer from Becoming Unavoidable}
Having established the key properties of the closed-loop, perfect information game solution, we can now reason about players' optimal strategies for the previous, open-loop stage.
Indeed, the winning conditions of this first game are directly related to the solution of the second.

In order to make the autonomous vehicle lose the overall game, the pursuer needs to drive the joint state into the capture basin of the closed-loop game without being previously detected: at this point even if the evader becomes aware of its state, the pursuer can enforce capture.
Conversely, if the evader can detect the pursuer before the state enters the capture basin, it can immediately apply its full-state feedback policy $\policy_e^t$ to avoid capture.

In this first stage, however, the evader has no knowledge of the pursuer's whereabouts, nor can it obtain any information about the pursuer until the pursuer's footprint enters its field of view.
As a result, the evader needs to blindly commit to an open-loop trajectory, which (in the worst case) will be met with an optimally adversarial initialization and execution of the pursuer trajectory.
We have the following results (proven in Appendix~\ref{app: proofs}):
\begin{proposition}[Hybrid game solution]\label{prop:hybrid_game}
    Let the closed-loop capture basin $\capturebasin: [0,T] \rightrightarrows \reals^{n_e} \times \reals^{n_\obj}$ be as given in~\eqref{eq:capture_basin}.
    Then, the evader starting from state $x_e^0$ can maintain safety and avoid capture by a hidden pursuer starting anywhere in $\hidden_\obj^0$ if and only if there exists an open-loop control signal $\csig_e$ for the evader such that $\forall t\ge 0$, $\footprint_e\big(\traj_e(t;0,x_e^0,\csig_e)\big) \cap \obst_e(t)=\emptyset$ and, additionally, for all initial pursuer states $ x_\obj^0\in\hidden_\obj^0$ and control signals $\csig_\obj$:
    \begin{align}\label{eq:hybrid_game}
        % \exists \csig_e \mid&\,
        % \forall x_\obj^0\in\hidden_\obj^0,
        % \forall\csig_\obj, \notag\\
        &\Big[\exists t\ge 0,\,
        \big(\traj_e(t;0,x_e^0,\csig_e),\traj_\obj(t;0,x_\obj^0,\csig_\obj)\big) \in \capturebasin(t)\Big] \implies\notag\\
        &\Big[\exists \tau\in [0,t):
        \footprint_\obj\big(\traj_\obj(t;0,x_\obj^0,\csig_\obj)\big) \in
        \fov\big(\traj_e(t;0,x_e^0,\csig_e),\obst_e(t)\big) \notag\\
        &\qquad\qquad\;\;\vee \footprint_\obj\big(\traj_\obj(t;0,x_\obj^0,\csig_\obj)\big) \in \obst_\obj(t)\Big]
    \end{align}
\end{proposition}
Proposition~\ref{prop:hybrid_game} states that, in order for the evader to avoid capture, it must be that any pursuer trajectory that eventually drives the state to the capture basin previously enters the evader's field of view or violates the pursuer's obstacle constraints.

The following result will serve as a bridge between the theoretical solution of the hybrid game and a safety-preserving onboard decision-making scheme for autonomous driving systems.

\begin{definition}[Forward Hidden Set]\label{def:fhs}
    The pursuer's forward hidden set at time $t$ for a given evader trajectory $\traj_e$ is the set of states that the pursuer could reach at time $t$ without violating its obstacle constraints or being discovered by the evader.
    \begin{equation*}
        \fhs_\obj(t;\hidden_\obj^0,\traj_e) = \fras_\obj\Big(\!t;0,\hidden_\obj^0,
        \!\footprint^{-1}_\obj\!\big(\fov\big(\traj_e(\cdot),\obst(\cdot)\big)\cup \obst_\obj(\cdot)\big)\!\Big) 
    \end{equation*}
\end{definition}

\begin{definition}[Danger Zone]\label{def:danger}
    The danger zone for a given evader state $x_e$ at time $t$ is the set of pursuer states from which the evader cannot safely avoid capture in the closed-loop game.
    \begin{equation*}
        \danger(t;x_e^t) = \big\{x_\obj\mid (x_e^t,x_\obj^t) \in \capturebasin(t)\big\}
    \end{equation*}
\end{definition}

\begin{theorem}[Forward-Backward Reach-Avoid Set Characterization]\label{thm:game_checker}
    The evader starting from state $x_e^0$ can maintain safety and avoid capture by a hidden pursuer starting anywhere in $\hidden_\obj^0$ if and only if there exists $\csig_e$ such that $\forall t\ge 0$, $\footprint_e\big(\traj_e(t;0,x_e^0,\csig_e)\big) \not\in \obst_e(t)$ and, additionally, for all $t\ge 0$
    \begin{equation*}\label{eq:game_checker}
        \fhs_\obj\big(t;\hidden_\obj^0,\traj_e(\cdot;0,x_e^0,\csig_e)\big)
        \cap
        \danger\big(t;\traj_e(t;0,x_e^0,\csig_e)\big)
        =\emptyset
    \end{equation*}
\end{theorem}

The above characterization provides a direct safety checking criterion for onboard decision-making.
If we can ensure that at all states in the autonomous vehicle's trajectory plan \emph{the forward hidden set and the danger zone are disjoint}, then the autonomous vehicle will be able to safely respond (through $\policy_e^t$) to maintain safety in the event that a hidden object appears.

\section{Technical Framework}\label{sec:framework}
Our theoretical results concern a two-mode hybrid dynamic game in continuous time. However, in real-world operation, autonomous vehicles receive sensor observations, conduct safety verification, and make decisions at discrete moments in time. Moreover, computing the winning condition for an infinite-horizon pursuit-evasion game is intractable for general cases. This section introduces an algorithmic trajectory planning framework based on the theoretical analysis in Section~\ref{sec:game} to overcome these challenges.

\subsection{Finite Horizon Approximation of the Closed-loop Game} \label{sec: close_loop_algo}

Robots, including autonomous vehicles, utilize receding horizon control \cite{Bellingham2002_receding_horizon} and partial motion planning \cite{petti2005_PMP} to compute partial trajectories efficiently and frequently replan to deal with uncertainty evolution over a long time. However, collision avoidance over a short planning horizon is not able to assure persistent feasibility and safety. Instead, previous works provide infinite horizon safety guarantees by planning partial trajectories within a \textit{safe set}
\cite{mayne2014_MPC_survey, schouwenaars2006_safeset}, or to avoid the \textit{inevitable collision set (ICS)} \cite{bouraine2014_passPMP, petti2005safe, wu2012guaranteed}. 

% We first define the fail set between the ego and another agent as a set of their joint collision states.
% \begin{equation}
% \mathcal{F}_{e}=\{(x_e^t, x_\obj^t)|\exists x_{\obj}^t\in\xset_\obj\text{ s.t. }\footprint_e(x_e^t)\cap \footprint_\obj(x_{\obj}^t)\neq\emptyset\}
% \label{eq:fail_set}
% \end{equation}

%% safe set
\begin{assumption}[Invariant Safe Set]\label{assu: invariant_safe_set} 
There exists a (possibly time-varying) \textit{invariant safe set} $\safeset: [0,\infty)\rightrightarrows\reals^{n_e}$ of collision-free states of the ego agent, from which, for any possible actions from other agents, the ego agent has a \textbf{known} strategy $\policy_e^\safeset: [0,\infty)\times \reals^{n_e}\to\cset_e$ that enables it to remain inside this set for an infinite time horizon:
\begin{subequations}\label{eq:safe_set}
\begin{align}
    % \safeset(t) = \{&(x^t_e, x^t_\obj) \mid  \forall\cstrat_\obj, \exists \csig_e, \forall\tau\geq t, \notag \\ 
    % & \footprint_e\big(\traj_e(\tau;x_e^t, \csig_e)\big) \cap \footprint_\obj\big(\traj_\obj(\tau;x_\obj^t, \csig_\obj)\big)=\emptyset \notag\\
    % & \wedge~ \footprint_e\big(\traj_e(\tau;x_e^t, \csig_e)\big)\cap\obst_e(\tau)=\emptyset\}
    &\forall x_e^t\in\safeset(t), \;
    \footprint(x_e^t)\cap \obst_e(t) = \emptyset,\quad \forall t\ge 0
    \label{eq:safe_set_noncollision}\\
    &\forall x_e^t\in\safeset(t), \forall x_\obj^t\in\hidden_\obj^t, \;
    (x_e^t,x_\obj^t)\not\in\capturebasin(t)
    \label{eq:safe_set_noncapture}\\
    &x_e^t \in \safeset(t) \implies
    \forall \tau\ge t, \traj_e\big(\tau; t, x_e^t, \policy_e^\safeset(\cdot)\big)\in\safeset(\tau)
    %  [0,\infty)\times \reals^{n_e}\to\cset_e
    \label{eq:safe set_invariance}
\end{align}
\end{subequations}
\end{assumption}

The control barrier function \cite{ames2019_cbf_survey} is introduced to verify and enforce such set. It can be approximated through optimization \cite{ames2016_cbf}, learned through demonstrations \cite{robey2020_cbf_learning}, and analytically expressed for specific dynamics \cite{pek2018_efficient_safeset}.

To retain an infinite horizon safety guarantee, we conservatively approximate the pursuit-evasion game in Section \ref{sec: closed-loop} as a finite horizon reach-avoid game, whose winning set for the evader is:
\begin{align} \label{eq:capture_basin_reach_avoid}
    \tilde{\capturebasin}^c(t) =
    \{&(x^t_e, x^t_\obj)\mid
    \forall \cstrat_\obj,
    \exists \csig_e,
    \exists\tau\in[t,T], \notag\\
    &\traj_{e}(\tau;t,x^t_e,\csig_e)\in\safeset,~ \wedge~ \forall s\in[t,\tau],\notag\\
    &\footprint_e\big(\traj_{e}(s;t,x^t_e,\csig_e)\big)\cap \footprint_\obj\big(\traj_\obj(s;t,x_\obj^t, \cstrat_\obj[\csig_e])\big) = \emptyset \notag\\
    &\wedge \footprint_e\big(\traj_e(s;t,x_e^t, \csig_e)\big)\cap\obst_e(s)=\emptyset \}
\end{align}
The ego agent must be able to reach a state in the invariant safe set before $T$ while avoiding capture by the pursuer and collisions with other objects $\obst_e$. After reaching a state in $\safeset$, the ego agent can apply a known feedback policy to achieve an infinite-horizon safety guarantee.
It can be seen by inspection of~\eqref{eq:capture_basin},~\eqref{eq:safe_set},~\eqref{eq:capture_basin_reach_avoid} that the
winning condition for the ego agent is an under-approximation of the complement of the original \textit{capture basin} defined in \eqref{eq:capture_basin},
i.e.~$\tilde{\capturebasin}^c(t)\subseteq\capturebasin^c(t)$.

Therefore, we are able to over-approximate the danger zone in Definition \ref{def:danger} using finite-time reach-avoid game, and maintain safety over infinite horizon. In general, the winning condition for the reach-avoid problem can be obtained through level-set methods \cite{mitchell2008flexible}, using Hamilton-Jacobi (HJ) reachability analysis \cite{Fisac2015ReachAvoidGame}, and applying sum-of-square optimization \cite{landry2018reach}.

\subsection{Open-loop Game with Discrete Sensor Observation} \label{sec: open_loop_algo}
Consider the ego agent receives new real-time sensor observation every $\Delta t$ seconds. At ${t_n} = n\Delta t$, the ego agent tries to determine whether it can safely move towards $x^{t_{n+1}}_e$. As shown in Theorem \ref{thm:game_checker}, the safety criterion is checking whether the forward hidden set $\hidden_\obj(t_{n+1}; \hidden_\obj^0, \traj_e)$ and the danger zone $\danger(x^{t_{n+1}}_e)$ are disjoint. As we do not have the access to future sensor observation, the forward hidden set can be over-approximated as \eqref{eq: forward_hidden_over} by the forward-reach avoid set of hidden set that avoiding being observed up to $t_n$.
\begin{multline}\label{eq: forward_hidden_over}
                \tilde{\hidden}_\obj(t_{n+1};\hidden_\obj^0,\traj_e) = \fras_\obj\Big(t_{n+1};0,\hidden_\obj^0,\\
                        \footprint^{-1}_\obj\big(\fov\big(\traj_e(0:t_n;\cdot),\obst(0:t_n;\cdot)\big)\cup \obst_\obj(\cdot)\big)\Big) 
    \end{multline}

In order to achieve the long term goal and provides an efficient motion strategy that minimizes any given cost function, the open-loop game in our framework plans trajectories in a receding horizon fashion. Without loss of generality, we consider the ego agent at state $x^{t_0}_e$ plans an open-loop trajectory with $N\Delta t$ seconds time-horizon. At each intermediate step, the planner tries to determine if a proposed sub-trajectory, which connects the waypoints $\tilde{x}^{t_{n-1}}$ to $\tilde{x}^{t_n}$, can be expanded to $\tilde{x}^{t_{n+1}}$. By assuming no new agent will emerge from the occlusion at $t_n$, it first predicts the states of all observed dynamic agents in the environment, and then estimates the field of view to predict the forward hidden set $\tilde{\hidden}_\obj(t_{n+1}; \hidden_\obj^0,\traj_e)$. If proposed waypoint $\tilde{x}^{t_{n+1}}$ satisfies the safety constraint defined in Theorem \ref{thm:game_checker}, the sub-trajectory is successfully expanded.

\zznote{}
With a long planning horizon, the open-loop game inexplicitly encourages the ego agent to choose safe trajectories that may lead to information gain and potentially open up more free space for planning by predicting the future information state optimistically. This allows the ego agent to plan trajectories, which  may originally be rejected by previous methods that avoid the entire FRS. On the other hand, the constraint introduced in Theorem \ref{thm:game_checker} also guarantees the existence of evasive trajectories when the actual sensor observation disagrees with the prediction or a new object emerged from the occlusion. Therefore, the proposed two-mode hybrid game framework provides sufficient conditions for safe trajectory planning in a partially occluded environment and can generate less conservative behaviors than previous methods.                                                               

\subsection{Framework of the Two-Mode Hybrid Game}
This subsection summarizes the proposed two-mode hybrid game in the Algorithm \ref{algo: twostage} for its application in trajectory planning. Our framework consists following functions:
\begin{itemize}
    \item \textbf{\texttt{Sense}}: detects obstacles, conducts state estimations and identify the field of view at a fixed rate.
    \item \textbf{\texttt{UpdateHidden}}: updates the forward hidden set described in Definiation \ref{def:fhs}.
    \item \textbf{\texttt{PlayOpenLoop}}: generates a receding-horizon open-loop trajectory described in Section \ref{sec: open_loop_algo}. 
     \item \textbf{\texttt{PlayClosedLoop}}: plans the evasive action to win the reach-avoid game introduced in the Section \ref{sec: close_loop_algo}.
     \item \textbf{\texttt{PlanSafe}}: checks the safety criterion in Theorem \ref{thm:game_checker}.
     \item \textbf{\texttt{InvSafeState}}: determines whether the ego agent wins the closed loop game by reaching $\safeset$.
\end{itemize}

Our framework operates under a finite state machine with two modes - \texttt{OpenLoop} and \texttt{ClosedLoop}. By assuming the vehicle starts inside ${\safeset}$, the system is initialized in the \texttt{OpenLoop} mode and frequently replan to improve efficiency. At each time step in the \texttt{OpenLoop} mode, the planned waypoint is verified based on the current observation using Theorem \ref{thm:game_checker}. If the safety verification failed, the system immediately switches to the \texttt{ClosedLoop} mode and plays the reach-avoid to avoid a collision and shield $\safeset$. Once the ego agent reaches a safe state, the system switches back to the \texttt{OpenLoop} mode and replan for future actions. 
\begin{algorithm}[ht]
	\caption{Occlusion-aware Trajectory Planning}
	\label{algo: twostage}
	\begin{algorithmic}[1]
		\renewcommand{\algorithmicrequire}{\textbf{Input:}}
		\renewcommand{\algorithmicensure}{\textbf{Output:}}
		\State \texttt{Mode} $\leftarrow$ \texttt{OpenLoop}
		\State $t \leftarrow 0$
		\State $t_{\text{replan}}\leftarrow 0$
		\State $\mathcal{I} \leftarrow \mathcal{C}_\obj$
		\While {running at the timestep $t$}
		\State $x_e, \fov, \obst_e \leftarrow$ \texttt{\textbf{Sense}()}
		\State $\hidden_\obj \leftarrow$ \texttt{\textbf{UpdateHidden}($\hidden,x_e,\fov$)}
		\If {\texttt{Mode} is \texttt{OpenLoop}} 
			\If {$t\geq t_{\text{replan}}$}
			    \State $\mathbf{x}\leftarrow$  \texttt{\textbf{PlayOpenLoop}($x,\obst_e,\hidden_\obj$)} 
				\State $t_{\text{replan}}\leftarrow t+T_{\text{replan}}$
			\EndIf
			\State $x_\text{plan}\leftarrow \mathbf{x}[t+1]$
			\If{not \texttt{\textbf{PlanSafe}(}$x, x_\text{plan}, \obst_e,\hidden_\obj$\texttt{)}}
			    \State $x_\text{plan}\leftarrow$ \texttt{\textbf{PlayClosedLoop}($x,\obst_e,\hidden_\obj$)}
			    \State \texttt{Mode} $\leftarrow$ \texttt{ClosedLoop}
			\EndIf
		\Else
		    \State $x_\text{plan}\leftarrow$ \texttt{\textbf{PlayClosedLoop}($x,\obst_e,\hidden_\obj$)}
		    \If{\texttt{\textbf{InvSafeState(}}$x_\text{plan}, \obst_e,\hidden_\obj$\texttt{)}}
			    \State \texttt{Mode} $\leftarrow$ \texttt{OpenLoop}
			    \State $t_{\text{replan}}\leftarrow t+1$
			\EndIf
		\EndIf
		\State \texttt{\textbf{DriveTo}(}$ x_\text{plan}$\texttt{)}
		\State $t \leftarrow t+1$
		\EndWhile
	\end{algorithmic} 
\end{algorithm}

\section{Simulation Results and Evaluation}\label{sec:results}
This section analyzes the benefit of modeling occlusion-aware trajectory planning problems as a two-mode hybrid game by comparing our approach with the previous work that over-approximates the hidden agent's occupancy through forward-reachability analysis. We describe our implementation details regarding our prototype planner and the resultant behaviors in three typical driving scenarios involving occlusion.  

\subsection{Dynamics and Trajectory Planner}
Our experiment uses the road centerline as the reference path and represents all agents' states in the Frenet frame, where $p$ is the agent's progress, and  $l$ is the lateral position relative to the road centerline. Their longitudinal and lateral velocities are denoted as $v_{lon}$ and $v_{lat}$. Assuming the agent's heading angle with respect to the reference path is small, we can represent the system's dynamics with a double integrator shown in the \eqref{eq: double_integrator}.
\begin{subequations}\label{eq:traj_experiment}
\begin{align}
\dot{p}(t) = v_{lon}(t),~~
\dot{v}_{lon}(t) = a_{lon}(t)\\
\dot{l}(t) = v_{lat}(t),~~~
\dot{v}_{lat}(t) = a_{lat}(t)
\end{align}\label{eq: double_integrator}
\end{subequations}
The autonomous vehicle is assumed to have direct control over its longitudinal and lateral accelerations denoted by $a_{lon}$ and $a_{lat}$. The accelerations of ego and hidden agents are subjected to constraints listed in the \eqref{eq: accel_constraint}.
\begin{equation}
    a_{lon}\in[-8,4]~\si{\meter\per\square\second} , ~~
    %a_{lat}\in[-8,8]~\si{\meter\per\square\second}\\
    \sqrt{a_{lon}^2+a_{lat}^2}\leq 8~\si{\meter\per\square\second}
\label{eq: accel_constraint}
\end{equation}
We construct a spatio-temporal lattice~\cite{mcnaughton2011_lattice, Ziegler2009Lattice} over a pre-defined route with $0.25~\si{\meter}$ spatial resolution in both longitudinal and lateral dimension. The open-loop game planner uses an A$^*$ search in a receding horizon fashion and generates a trajectory satisfying winning conditions stated in Theorem~\ref{thm:game_checker} and maximizing the progress at the end of the time horizon. 

At each step, we conduct a ray-casting computation to measure the current field of view and use ground truth data from CARLA for observed objects' state. During the planning of the open-loop game, we predict the field of view from a pre-constructed Octomap~\cite{hornung13_octomap} and estimate the occlusion due to dynamic obstacles from their predicted occupancy. Following the method proposed in \cite{neel2020improving}, the forward hidden set $\fhs^t_\obj$ is approximated by the velocity upper bound of hidden agents based on the occluded region's length and observation history.
\zznote{} Although obtaining the solution for the closed-loop reach-avoid game is computationally challenging in general case, our prototype planner takes the advantage of the simplified dynamics and approximates the closed-form solution through a finite set of motion primitives, as detailed in Appendix~\ref{app: reach-avoid-approx}.

While our prototype planner utilizes an Octomap to predict future visibility, a precise representation of the environment is not required for safety guarantees. Our framework always verifies the safety of the next step in a planned trajectory before executing it, by checking the Theorem~\ref{thm:game_checker} condition for the current sensor observation and updated hidden set. In the event that the verification is negative, a closed-loop strategy still exists to avoid the collision under our safety criterion. However, accurate prediction of future visibility can improve planning efficiency and avoid frequent evasive behaviors.

\subsection{Invariant Safe Set}
%s we have discussed in the Section~\ref{sec: close_loop_algo}, a finite-time reach-avoid game that reaching a state in the invariant safe set $\safeset$ and avoiding collisions to achieve a infinite horizon safety guarantee. 
In our experiments, we assume the ego agent is safe when its entire footprint occupies a single lane with correct traffic direction, and the vehicle is either stopped or keeps a safe distance with agents ahead. The safe distance can be obtained by determining the minimum required distance to avoid collisions, when both the ego agent and other traffic participants apply the maximum deceleration for emergency brake. However, stopping in opposite lanes or in the intersection is not considered safe.

\subsection{Scenario 1: Blind Summit}
\begin{figure}[t!]
    \centering
    \includegraphics[width=3.5in]{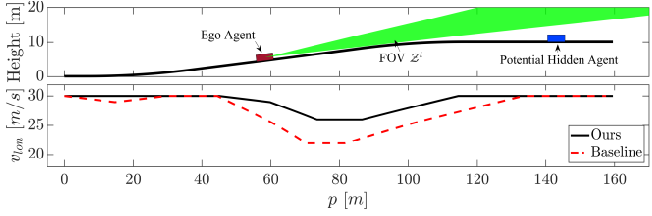}
    \captionsetup{belowskip=-10pt}
    \caption{The road geometry and ego vehicle's progress versus its longitudinal velocity in the blind summit example}
    \label{fig:blind_summit_v_profile}
\end{figure}
In the first experiment shown in Fig~\ref{fig:blind_summit_v_profile}, we consider the ego vehicle approaches an uphill slope in the CARLA map \texttt{Town05}. The initial velocity of the ego agent is $30~\si{\meter\per\second}$. The hill has a $7.2^\circ$ slope with $10~\si{\meter}$ vertical difference. The planning horizon of the open-loop game is $10$ steps, $0.5~\si{\second}$ apart. The open-loop trajectory is replanned every time step. 
By considering the worst-case scenario of a potential obstacle stopped in the occluded region, the closed-loop game reduces to determining whether the ego agent can stop before colliding with the hidden obstacle. As shown in Fig.~\ref{fig:blind_summit_v_profile}, our proposed method, which considers the vehicle's future sensing and decision-making ability, allows the ego agent to plan a trajectory with higher velocity and efficiency when compared with the baseline method.

\subsection{Scenario 2: Intersection With Partial Occlusion}
In the second experiment as shown in Fig.~\ref{fig:intersection_scene}, the ego agent drives in an urban environment of CARLA map \texttt{Town01} and approaches a T intersection 60 m ahead, where a part of oncoming traffic is occluded by parked trucks. We set the speed limits as 15 $\si{\meter\per\second}$ for all agents in the scene. The planning horizon of the open-loop game is $6$ steps, 0.5 $\si{\second}$ apart. The open-loop game is replanned every time step to improve planning efficiency.
\begin{figure}[ht]
    \centering
    \includegraphics[width=3.5in]{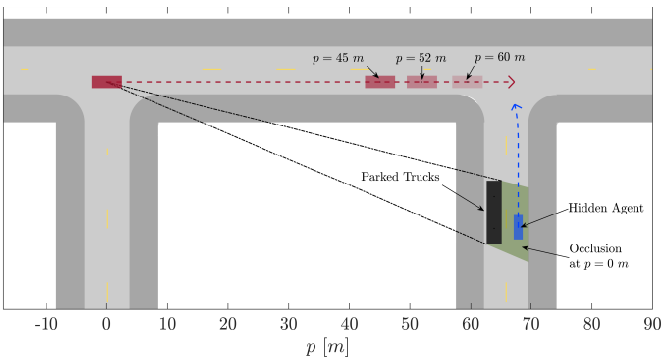}
    \captionsetup{belowskip=-10pt}
    \caption{The ego agent approaches an open T intersection, with the oncoming lane partially occluded by parked trucks.}
    \label{fig:intersection_scene}
\end{figure}

\begin{figure}[ht]
    \centering
    \includegraphics[width=3.5in]{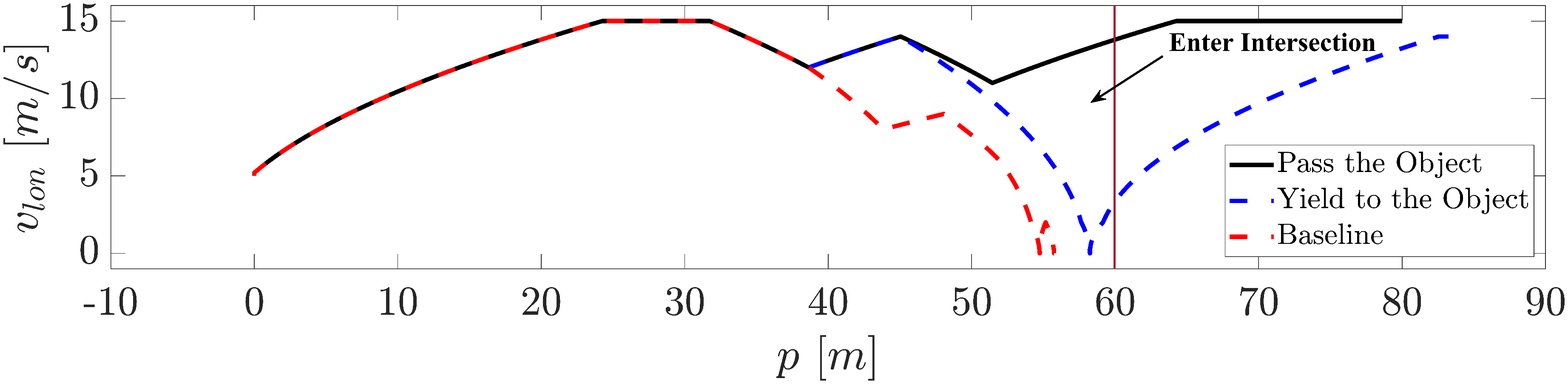}
    \captionsetup{belowskip=-4pt}
    \caption{Ego agent's progress versus its longitudinal velocity in the blind T intersection example}
    \label{fig:intersection_v_compare}
\end{figure}

\begin{figure*}[ht]
	\centering
	\begin{subfigure}[b]{0.3\textwidth}
		\centering
		\includegraphics[height=1.8in]{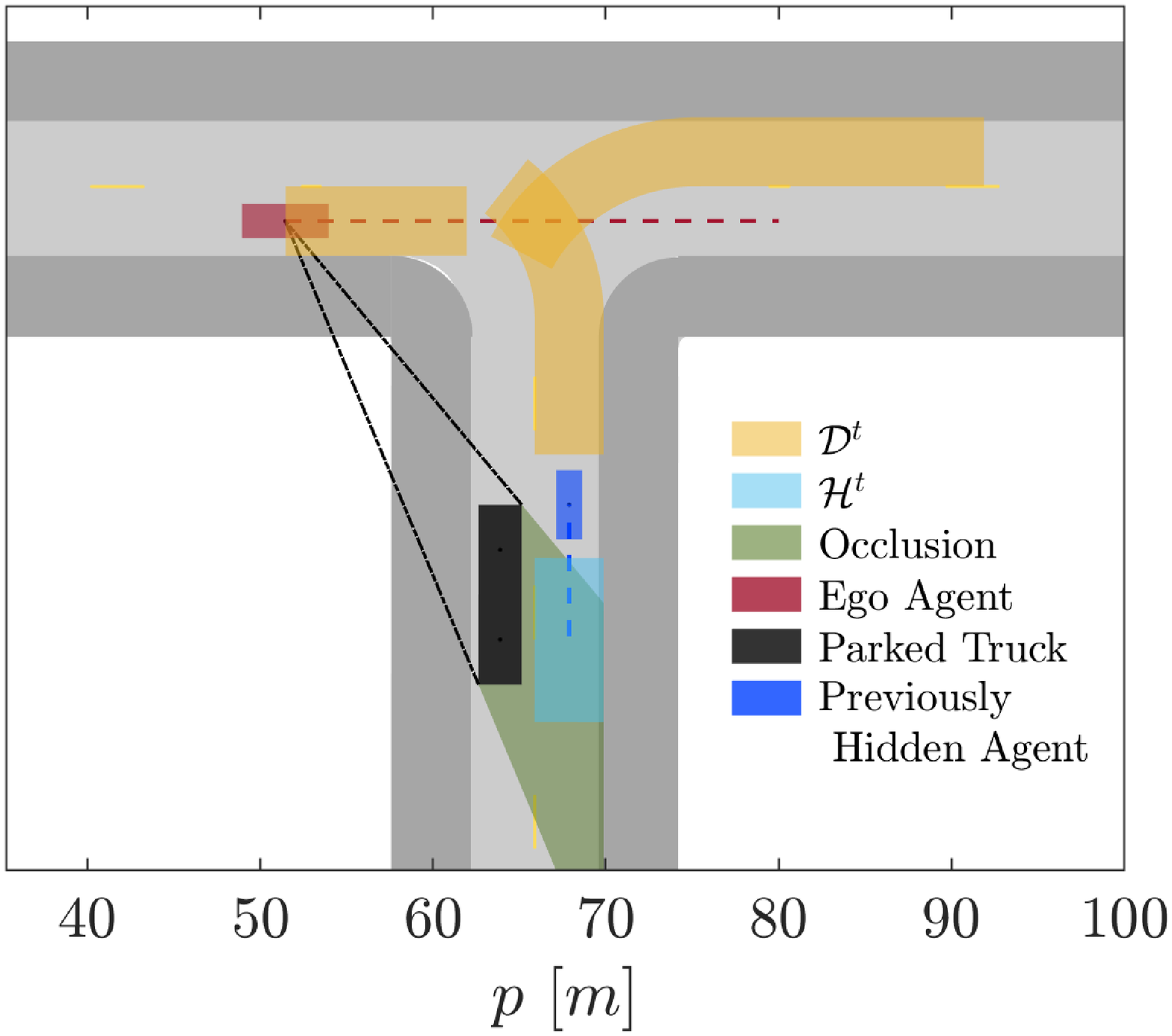}
		\caption{$p=52~\si{\meter}$ in the first scenario}
		\label{fig: intersection_pass_52}
	\end{subfigure}	
\hfill
	\begin{subfigure}[b]{0.3\textwidth}
		\centering
		\includegraphics[height=1.8in]{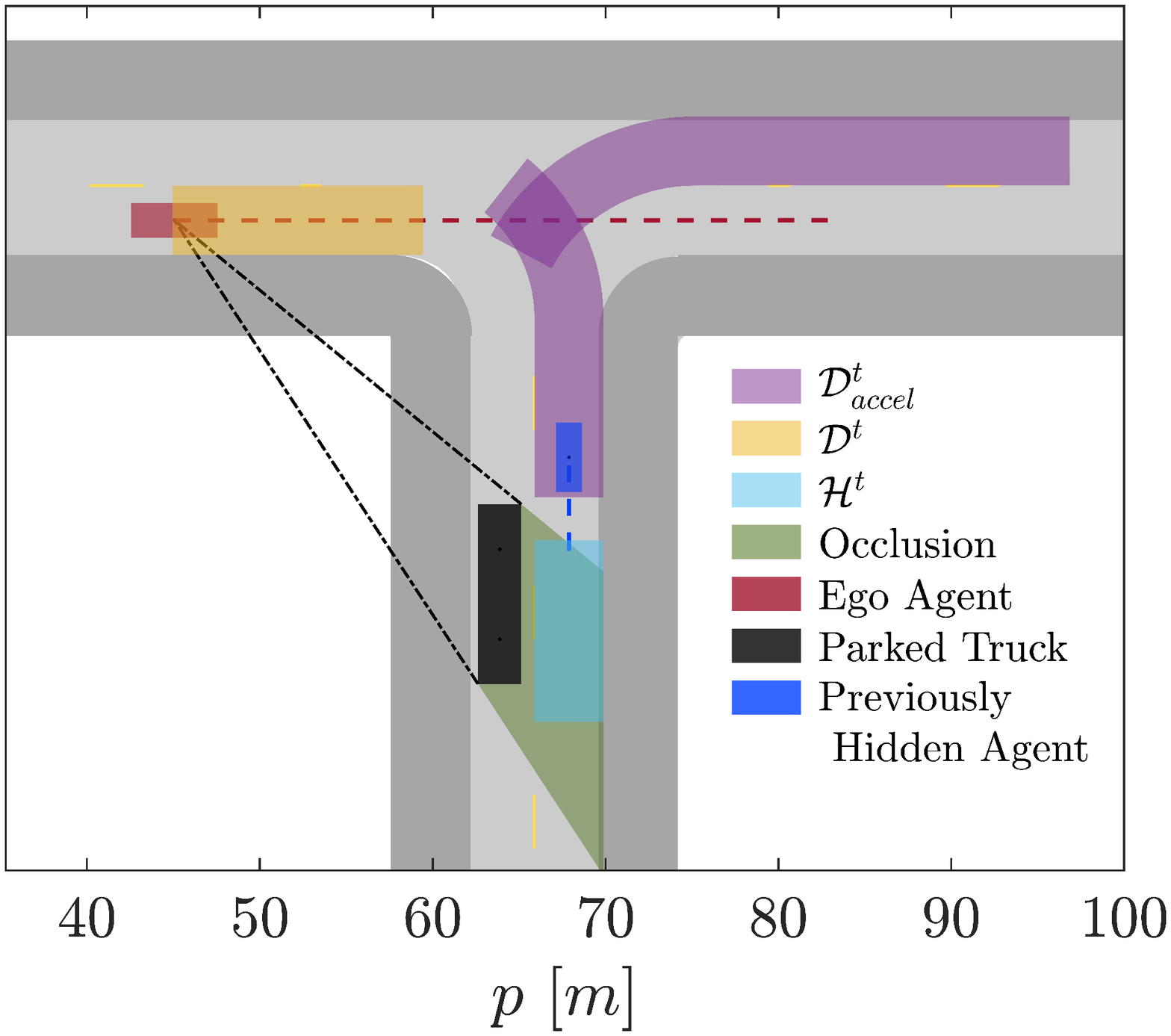}
		\caption{$p=45~\si{\meter}$ in the second scenario}
		\label{fig: intersection_yield_45}
	\end{subfigure}	
\hfill 
	\begin{subfigure}[b]{0.3\textwidth}
		\centering
		\includegraphics[height=1.8in]{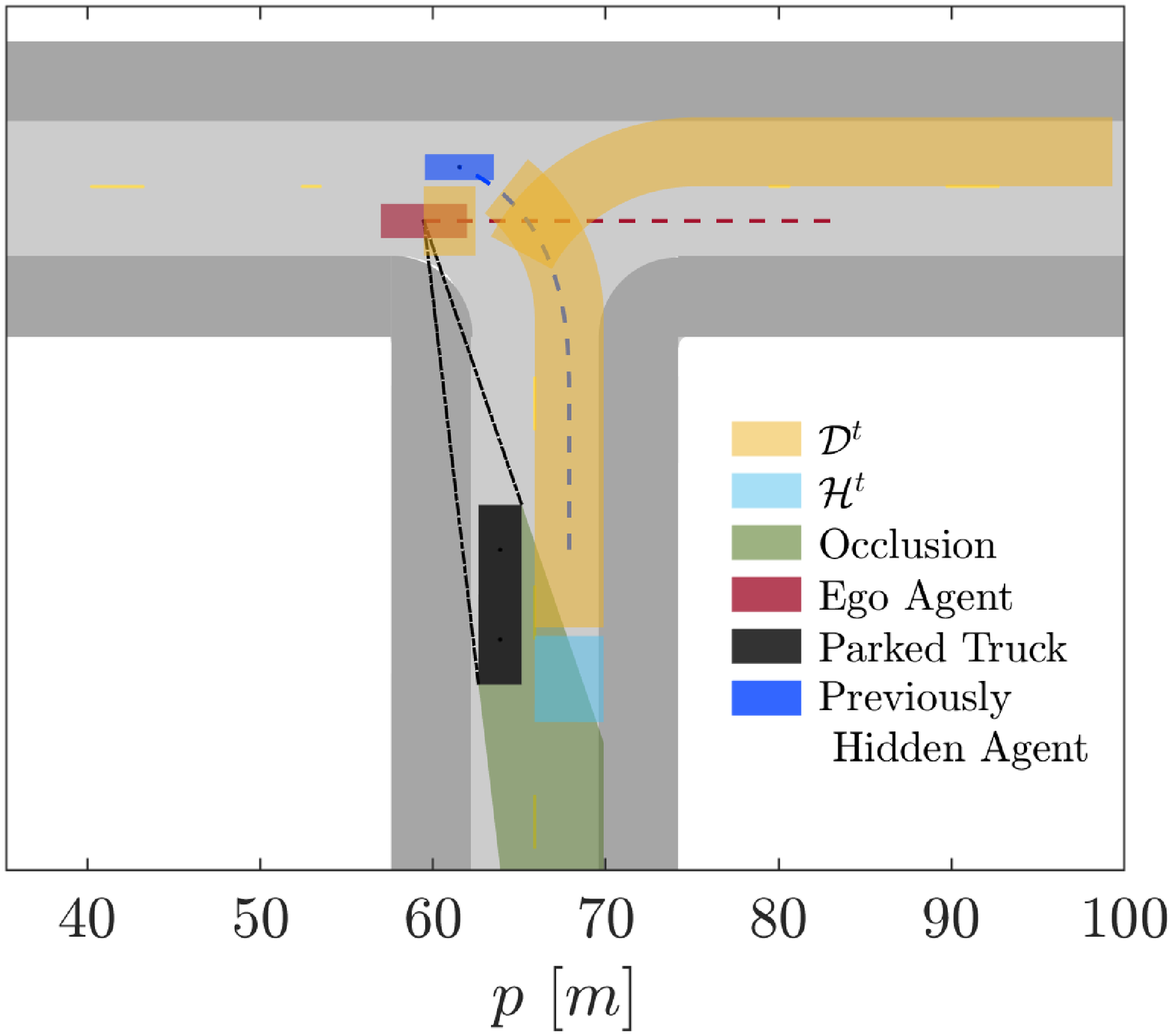}
		\caption{$p=60~\si{\meter}$ in the second scenario}
		\label{fig: intersection_yield_60}
	\end{subfigure}	
\captionsetup{belowskip=-10pt}
\caption{Three distinctive interaction scenarios in our intersection example. a) the ego agent approaches the intersection with the pursuer outside its danger zone. b) the ego agent approaches the intersection and starts slowing down as the pursuer enters the intersection. c) the ego agent continues to proceed after the pursuer exited the intersection}
\label{fig: intersection_scenario}		
\end{figure*}

Assuming that the ego vehicle remains in its lane, the closed-loop game solution in this example is analytically obtainable, with the optimal evasive policy always consisting in one of two ``bang-bang'' actions. The vehicle can apply the maximum deceleration to either enter the intersection as late as possible or stop before the junction. Alternatively, the vehicle can apply the maximum acceleration to pass the intersection as early as possible.
%For each action, we first calculate the backward reachable set of other agents that will inevitably lead to collisions. The danger zone can be determined by taking the intersection of two backward reachable sets.

We test two intersection examples with hidden objects entering the field of view at different times. The longitudinal velocity profile of the proposed algorithm under each scenario and the velocity profile of the baseline method are shown in Fig.~\ref{fig:intersection_v_compare}. In the first example, the ego agent initially slows down and then maintains its speed as it approaches the intersection. At $p=52~\si{\meter}$, as shown by the same-lane danger zone in Fig.~\ref{fig: intersection_pass_52}, the ego vehicle would be able to stop just before entering the junction. However, as the danger zone does not intersect with the recently detected agent, the ego vehicle can safely traverse the intersection by applying full acceleration.

In the second example, the hidden agent is set to enter the field of view earlier than the previous case. As shown in Fig.~\ref{fig: intersection_yield_45}, the ego agent detects the previously hidden agent at $p=45~\si{\meter}$ and the safety condition indicates the (would-be) pursuer will enter the danger zone if the ego agent attempts to accelerate. Therefore, the ego agent takes an evasive trajectory and stops before the junction. Once the agent is about to exit the intersection, the ego vehicle slowly begins to enter it, and then proceeds by acceleration when the traffic is clear. The planner benefits from our two-mode game framework since the analysis predicts that the ego vehicle will gain sufficient visibility to safely pass the intersection. As shown in Fig.~\ref{fig: intersection_yield_60}, at  $p=60~\si{\meter}$  when the ego vehicle enters the intersection from static, the danger zone is barely disjoint with the hidden set. In the baseline result, the ego vehicle always freezes before the intersection due to the rapidly growing forward-reachable set of hidden states over the planning horizon (uninhibited by future measurements).

%If the vehicle can be safely stopped before the intersection, it will successfully win the closed-loop game by following this stopping trajectory. On the other hand, if being static before the intersection is impossible, the ego agent can estimate the earliest and latest times entering and exiting the intersection by applying constant maximum deceleration or acceleration.  Then we can use backward reachability analysis to obtain the danger zone in the intersection example. Thus, this example's danger zone can be over-approximated as the intersection between two backward reachable sets. If the pursuer's current state does not intersect with the danger zone, the ego agents will have a strategy to avoid colliding with the pursuer and win the closed-loop game. 

\subsection{Scenario 3: Truck Overtaking}

\begin{figure}[ht]
    \centering
    \includegraphics[width=3.5in]{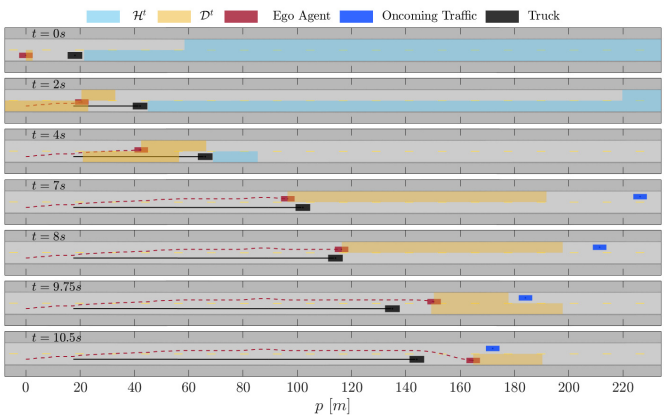}
    \captionsetup{belowskip=4pt}
    \caption{The ego agent safely overtakes the truck, merges back to the driving lane, and avoids collisions with oncoming agent }
    \label{fig:overtake_finish}
\end{figure}

\begin{figure}[ht]
    \centering
    \includegraphics[width=3.5in]{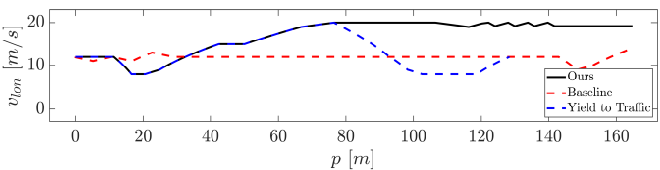}
    \captionsetup{belowskip=-8pt}
    \caption{Ego agent's progress versus its longitudinal velocity during the overtaking attempts}
    \label{fig:overtake_vel}
\end{figure}
\begin{figure}[ht]
    \centering
    \includegraphics[width=3.5in]{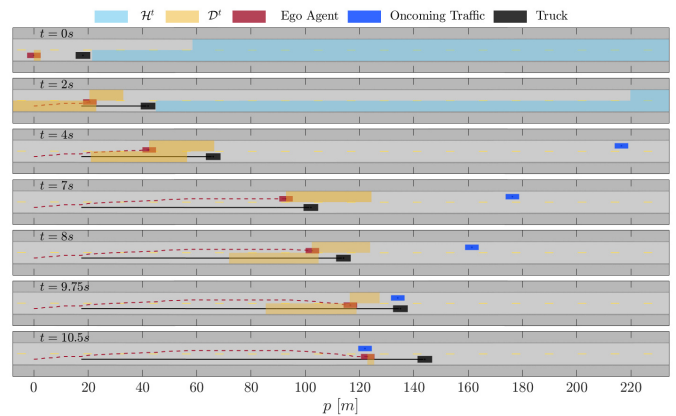}
    \captionsetup{belowskip=-10pt}
        \caption{The ego agent attempts to overtake the truck and safely aborted the overtaking after detecting the oncoming agent}
    \label{fig:overtake_yield}
\end{figure}

 Finally, we set up the overtaking example in a long undivided highway of CARLA map \texttt{Town01} as shown in Fig.~\ref{fig:overtake_finish}. The ego agent is initially $10 ~\si{\meter}$ behind the truck, and both vehicles start with the velocity of $12~\si{\meter\per\second}$. The speed limit in both traffic directions is $20~\si{\meter\per\second}$. The open-loop game's planning horizon is $15$ time steps, $0.5~\si{\second}$ apart. Trajectories are replanned every time step to improve planning efficiency.

We formulate the closed-loop strategy by first finding two shortest evasive trajectories reaching invariant safe states behind and ahead of the truck.
Since our vehicle dynamics are represented as a double integrator, both trajectories and danger zone can be analytically expressed. 

In Fig.~\ref{fig:overtake_finish} and~\ref{fig:overtake_vel}, we show the ego agent's overtaking trajectory and corresponding velocity profile using the proposed algorithm. Initially, the autonomous vehicle slows down to increase the distance behind the truck and gain visibility of oncoming traffic. After the lane change, the vehicle briefly accelerates and then maintains a constant velocity to observe the driving lane ahead of the truck. By doing so, the ego agent maintains the option to return to a safe state behind the truck, in the event of an oncoming vehicle being observed later. Once it determines that it is safe to merge in the front of the truck, the ego agent accelerates immediately, completes the overtaking, and then returns back to its original lane. During this process, even though an oncoming agent enters the field of view, the ego agent still successfully finishes the overtaking since the agent never enters the danger zone. For comparison, we test the baseline algorithm in the same scenario: velocity profiles are shown in Fig.~\ref{fig:overtake_vel}. The vehicle using the baseline method-keeps a constant velocity behind the truck, because the forward reachable set of hypothetical oncoming traffic limits its ability to change lanes. 

In another overtaking example shown in Fig.~\ref{fig:overtake_yield}, we set the oncoming traffic to enter the ego agent's field of view earlier than the previous scenario. As shown in the velocity profile of Fig.~\ref{fig:overtake_vel}, the ego agent follows the same strategy initially and moves towards the opposite lane. After the oncoming traffic is detected, the ego agent determines it is unsafe to overtake the truck. It slows down immediately, merges back into the driving lane, and avoids colliding with the oncoming vehicle.

\section{Conclusion and Future Work}\label{sec:conclusion}
In this paper, we introduce a novel game-theoretic analysis for autonomous vehicles' motion planning under densely-occluded environments, and derive a decision-making framework that is compatible with a wide class of motion planning algorithms and provides worst-case safety guarantees. We demonstrate our approach in the CARLA simulator.

This work opens the door to multiple promising research directions. In real driving settings, the autonomous vehicle often encounters situations that require interaction and coordination with multiple agents to navigate efficiently and safely. Although obtaining a closed-loop solution for dynamic games involving multiple agents is computationally intractable in real time, the proposed framework may become a building block to efficiently approximate solutions for such multi-sided encounters by combining open- and closed-loop strategies from multiple sub-games.

While our implementation assumes perfect perception and state estimation, the game-theoretic analysis can be extended to scenarios in which sensing limitations are induced not only by obstacles but also by the physical properties and eventual failures of the sensor itself. When incorporating a more sophisticated sensor model, our hybrid game setup may help prevent unsafe behavior due to perception uncertainty and sensor malfunctions.
 
Finally, while our game-theoretic analysis computes safety measures under worst-case scenarios, probabilistic modeling of hidden agents can be built upon the current framework to generate more efficient vehicle motion, by considering the observed traffic density and other prior or acquired knowledge about traffic participants' behavior.

%% Use plainnat to work nicely with natbib. 

\bibliographystyle{unsrtnat}%{plainnat}
\bibliography{references}

\clearpage
\appendix
\subsection{Proofs of Theoretical Results}\label{app: proofs}
\begin{proof}[Proof of Proposition~\ref{prop:hybrid_game}]
    The proof of this proposition is constructive, by enumerating the winning and losing conditions for both players.
    The evader's winning conditions in the hybrid game are those from which the it can indefinitely avoid capture by the pursuer while remaining collision-free. Since the evader will not obtain new information about the pursuer's state until detection, it must (at least initially) commit to an open-loop strategy $\csig_e$. Because we are concerned with worst-case analysis, the pursuer can choose its initial state $x_\obj^0$ and open-loop strategy $\csig_\obj$ \emph{after} the evader has declared $\csig_e$.

    \textbf{P1: evader collision.} An evader control signal $\csig_e$ can only be a winning strategy if the resulting trajectory avoids collisions with environment obstacles $\obst_e(t)$ for all future time. Failure of this condition at any future time means that the evader cannot maintain safety and therefore loses the game.
    
    \textbf{P2: capture basin reached.} If at some time $t\ge 0$ the state enters $\capturebasin(t)$, however, then by definition the pursuer can subsequently achieve capture in the closed-loop game; this further implies that the pursuer can achieve capture \emph{regardless} of whether or not the evader eventually detects it. Indeed, at time $t$, the pursuer has an available non-anticipative strategy $\cstrat_\obj$ for which capture ensues for \emph{all} possible evader control signals, including the currently declared $\csig_e$. (Conversely if the joint open-loop trajectory never enters the capture basin of the closed-loop game, then the evader's open-loop strategy is guaranteed to avoid capture regardless of whether or not the pursuer is detected.)
    
    The only way for the evader to win the game if the joint open-loop trajectory enters the capture basin at time $t$ is for its winning conditions to take place at some $\tau\in[0,t)$.
    
    \textbf{E1: pursuer collision.}
    If the pursuer has violated its obstacle constraints $\obst_\obj(\tau)$, it has lost the game and its subsequent open-loop trajectory is irrelevant (and physically meaningless).
    
    \textbf{E2: pursuer detected.}
    If the pursuer has been detected by entering the evader's field of view while the joint state is outside of the capture basin $\capturebasin(\tau)$, the remainder of the open-loop trajectory is not realized, and the game instead transitions to the closed-loop mode, in which the evader is guaranteed to avoid capture and remain collision-free by applying a feedback strategy.
    
    Since the above classification of game executions is complete (every possible game outcome belongs to one of the above categories), we have that the evader wins the game if, and only if, $\neg \mathbf{P1} \wedge \big(\mathbf{P2}\implies(\mathbf{E1} \vee \mathbf{E2})\big)$. This concludes the proof.
\end{proof}

\begin{proof}[Proof of Theorem~\ref{thm:game_checker}]
    We need to show that, for any given $\csig_e$, characterization~\eqref{eq:game_checker} imposed for all $t\ge 0$ is equivalent to condition~\eqref{eq:hybrid_game}.
    Indeed, let us require that~\eqref{eq:game_checker} holds $\forall t\ge 0$, that is, the forward hidden set $\fhs_\obj(t;\hidden_\obj^0,\traj_e)$ and danger zone $\danger(t;x_e^t)$ are disjoint for all $t$.
    By Definition~\ref{def:fhs}, this condition is equivalent to requiring that any states that the pursuer can reach at time $t$ without first being detected or violating the obstacle constraints must be outside of the danger zone, that is, by Definition~\ref{def:danger}, the joint state $(x_e^t,x_\obj^t)\not\in\capturebasin(t)$.
    This requires that there can be no undetected, collision-free trajectory of the pursuer for which the joint state eventually enters the closed-loop capture basin, which is exactly the contrapositive of~\eqref{eq:hybrid_game}.
    This concludes the proof.
\end{proof}

\subsection{Approximate reach-avoid solution} \label{app: reach-avoid-approx}
To further simplify the closed-loop game, we can under-approximate the winning condition of the reach-avoid problem defined in \eqref{eq:capture_basin_reach_avoid} by searching among a finite set of open-loop strategies $\mathbf{U}=\{\csig_1, \csig_2,\hdots\}$.  

We first define a set of reaching trajectory $\tilde{\traj}_i$, which drives the ego agent to a state in the target invariant safe set $\safeset$ before $T$ and avoid colliding with other objects using one of the open-loop strategy $\csig_i$.
\begin{align}
    \mathcal{X}_e(x^t_e, T,\mathbf{U}) =\big\{&\tilde{\traj}_i(\cdot;t, x^t_e, \csig_i) ~|~ \exists\csig_i\in\mathbf{U},~\exists\tau\in[t,T], \notag\\
    &\tilde{\traj}_i(\tau; t, x_e^{t}, \csig_i)\in \safeset \wedge \forall s\in[t,\tau],\notag\\ 
    &\footprint_e(\tilde{\traj}_i(\tau; t, x_e^{t}, \csig_e))\cap\obst_e(s) = \emptyset\big\}
\end{align}

Assume the ego agent is committing to one of the reaching trajectories $\tilde{\traj}_i$, the corresponding trajectory-specific danger zone $\bar{\danger}(\tilde{\traj}_i)$ is defined as the set of current pursuer states from which the pursuer can achieve capture against this ego trajectory within the time horizon $T$.
\begin{align}
    \bar{\danger}(\tilde{\traj}_i) = \{&x_\obj^t|\exists \cstrat_\obj,\exists \tau\in[t,T],\notag\\
    &\footprint_e( \tilde{\traj}_i(\tau; t, x_e^{t}, \csig_e))\cap\footprint_\obj( \tilde{\traj}_\obj(\tau; t, x_\obj^{t}, \cstrat_e))\neq \emptyset
\end{align}

Therefore, we can over-approximate the danger zone $\danger(x_e^t)$ of this reach-avoid game as
\begin{equation}
    \danger(x_e^t) = \bigcap_{\tilde{\traj}_i\in\mathcal{X}_e(x^t_e, T,\mathbf{U})}\bar{\danger}(\tilde{\traj}_i).
    \label{eq: danger_zone_reach_avoid}
\end{equation}

\end{document}